\documentclass{article}

\usepackage[numbers]{natbib}

\usepackage[preprint]{neurips_2026}

\usepackage[utf8]{inputenc} 
\usepackage[T1]{fontenc}    
\usepackage{hyperref}       
\usepackage{url}            
\usepackage{booktabs}       
\usepackage{amsfonts}       
\usepackage{nicefrac}       
\usepackage{microtype}      
\usepackage[dvipsnames]{xcolor}

\usepackage{algorithm}
\usepackage{algpseudocode}
\usepackage{float}
\usepackage{graphicx}
\usepackage{subcaption}
\usepackage{bm}
\usepackage{mathtools}
\usepackage{amsmath}
\usepackage{amssymb}
\usepackage{mathtools}
\usepackage{amsthm}
\usepackage{comment}

\floatstyle{plaintop}
\restylefloat{table}
\usepackage[tableposition=top]{caption}

\theoremstyle{plain}
\newtheorem{theorem}{Theorem}
\newtheorem{proposition}[theorem]{Proposition}

\newtheorem{corollary}{Corollary}
\theoremstyle{definition}

\theoremstyle{remark}

\title{Physiological Noise Augmentation \\ Improves Non-Invasive Brain-to-Speech}

\author{%
    Benjamin Ballyk\thanks{Equal contribution.} \quad
    Teyun Kwon\footnotemark[\value{footnote}] \quad
    Miran {\"O}zdogan \quad
    Oiwi Parker Jones\\
    \\
    PNPL\includegraphics[height=1.5\fontcharht\font`\B]{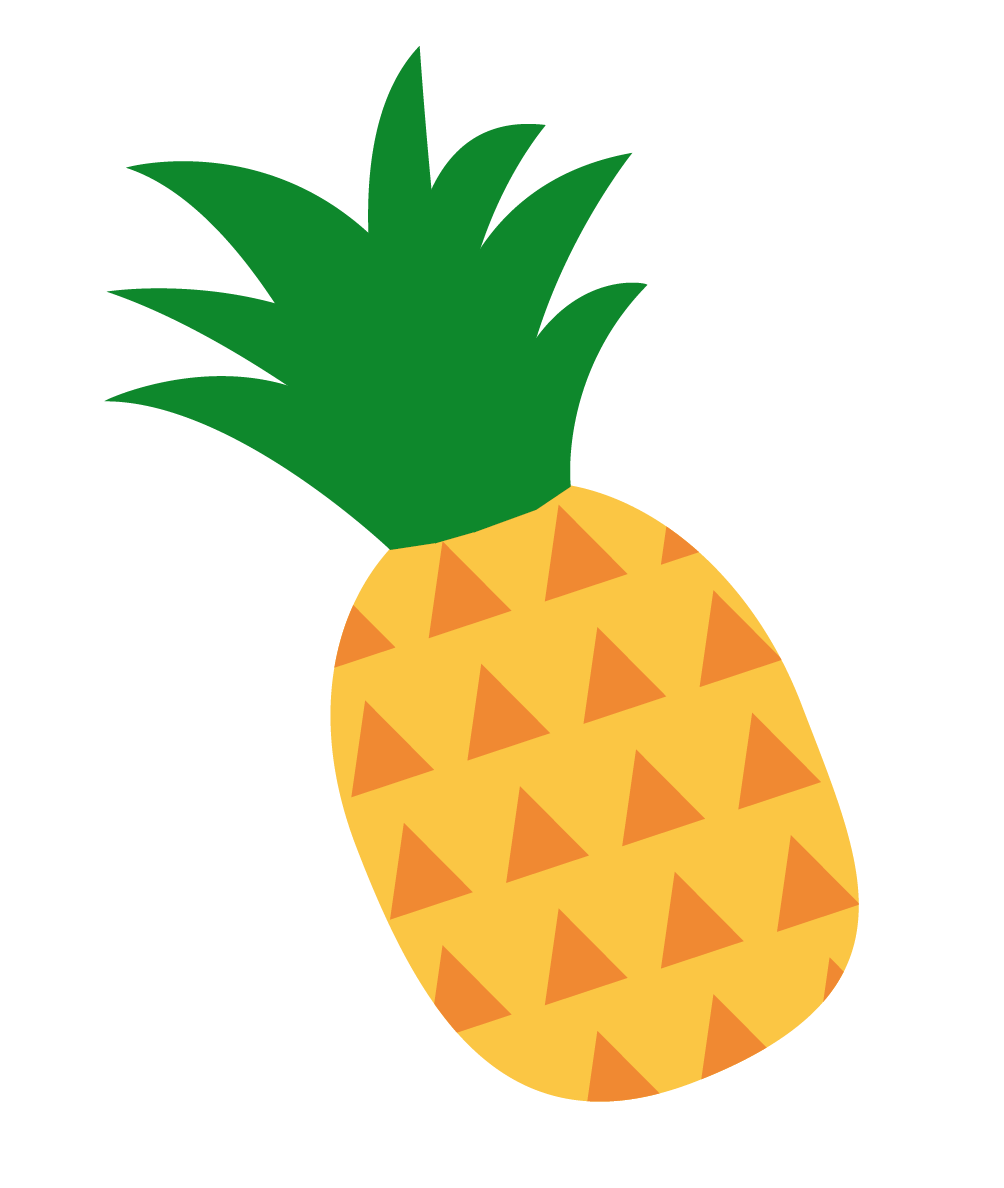}\\
    Department of Engineering Science\\
    University of Oxford\\
    \texttt{\{benjamin.ballyk, teyun.kwon\}@eng.ox.ac.uk}
}

\begin{document}

\maketitle

\begin{abstract}
\looseness=-1 Non-invasive brain-to-speech decoding aims to restore communication to patients suffering from neurodegenerative disease, without the risks of neurosurgery. Existing MEG- and EEG-based methods, while scalable, continue to suffer from high word error rates driven by relatively low signal-to-noise ratios compared to invasive recordings. We propose \textit{physiological noise augmentation} (PNA), a data augmentation method that explicitly trains decoders to become invariant to task-agnostic artifacts (e.g. ocular and cardiac activity). PNA draws inspiration from automatic speech recognition systems, where environmental noise (e.g. dogs barking, city traffic) is added to clean speech to improve robustness. Analogously, we decompose brain recordings into clean data and noise artifacts using independent component analysis (ICA), before scaling and remixing to generate biophysically realistic, label-preserving training examples. We show that PNA approximates anisotropic regularization, penalizing decoder sensitivity along artifact-dominated directions. On MegNIST, a 12k-trial imagined-digit MEG dataset, PNA with 10-trial averaging improves EEGNet decoding accuracy by $4.7$ percentage points (absolute) over training on real data alone. Our results suggest that artifact-aware augmentation and trial averaging are complementary tools for improving robustness in non-invasive speech BCIs.
\end{abstract}

\section{Introduction}

Neurodegenerative disease, stroke, and cervical spine injuries collectively affect more than 110 million patients globally, and often irreversibly impair one’s ability to articulate thoughts into speech \cite{feigin2021global, injury2019global, park2022global}. For nearly four decades, brain-to-speech research has sought to restore communication for these patients, yet this goal remains a central challenge at the intersection of neuroscience and biomedical engineering \citep{farwell1988talking, Birbaumer1999, wolpaw2002bci}. Although recent intracranial brain-computer interfaces (BCIs) have achieved increasingly accurate decoding of intended speech from cortical activity \citep{moses2021nejm, willett2023nature, card2024nejm}, these invasive systems carry risks of neurosurgical complications and long-term electrode instability \citep{leuthardt2021risk}.

Consequently, there is a growing imperative to develop safer and more accessible non-invasive alternatives based on magnetoencephalography (MEG), electroencephalography (EEG), or functional magnetic resonance imaging (fMRI) \citep{wolpaw2002bci, dascoli2025nature}. Non-invasive brain recordings, however, suffer from low signal-to-noise ratio (SNR), particularly in the context of imagined speech, which lacks overt articulatory and auditory feedback \citep{pei2011jne, martin2016inner, panachakel2021decoding, martin2014front, martin2016scirep}. For several decades, this limitation has been addressed using multi-trial averaging, in which repeated recordings of the same thought are averaged to enhance phase-locked signal relative to independent noise \citep{sutton1965p300}. However, this approach is often slow and burdensome, as users must repeat themselves several times, which limits practical applicability.

Concurrently, data collection bottlenecks have motivated augmentation strategies to expose decoding models to a broader range of realistic test-time conditions. While often modestly beneficial, existing augmentations typically operate at the input signal level rather than targeting underlying neurophysiological artifacts, leaving models exposed to structured noise that dominates imagined speech signals \citep{luo2021channel, he2021data}.

\begin{figure}[t]
    \centering
    \includegraphics[width=\linewidth]{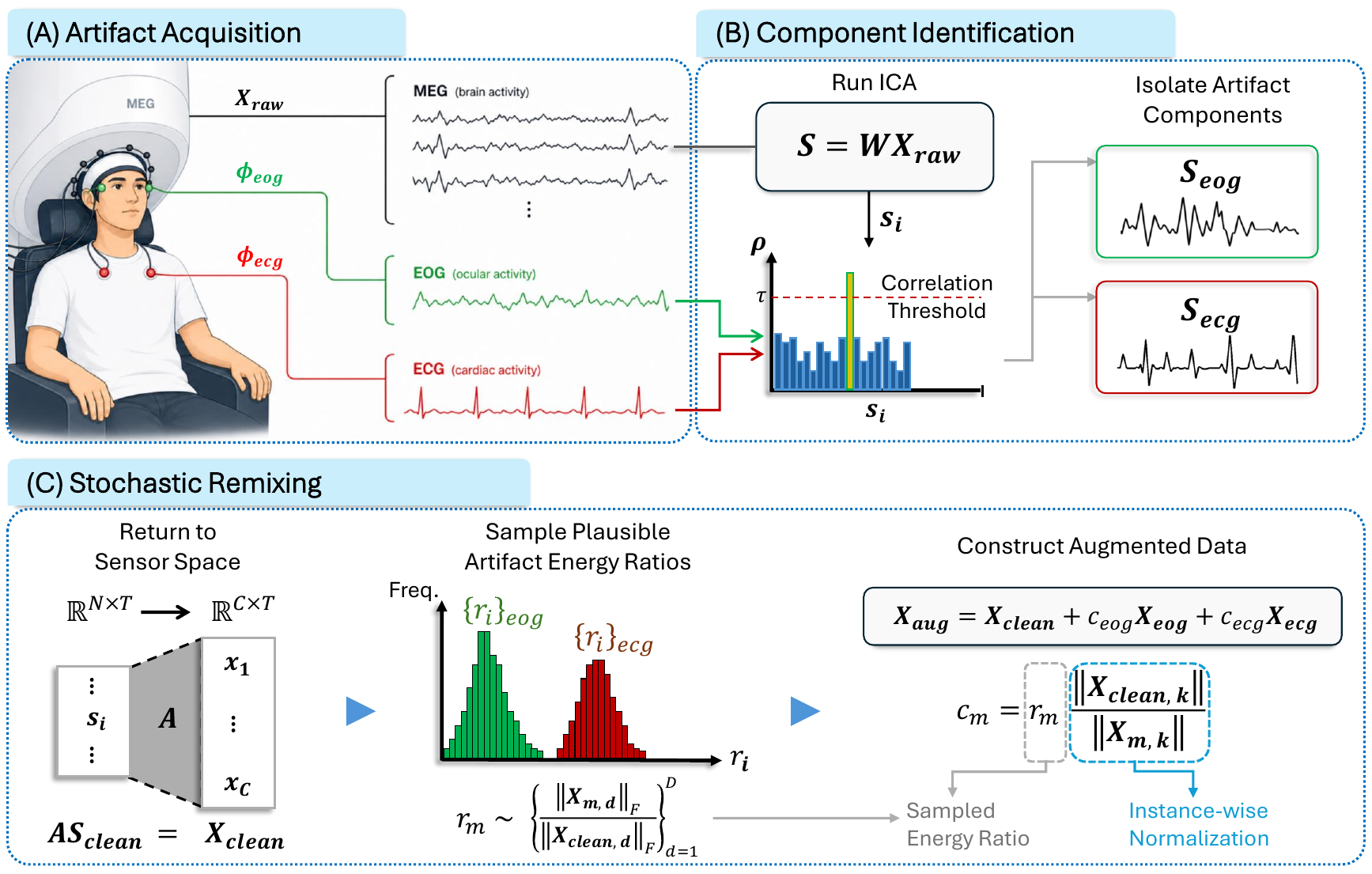}
    \caption{\textbf{Overview of physiological noise augmentation}. (A) Raw brain recordings ($\bm{X}_{\mathrm{raw}}$) and reference channels (e.g. ocular ${\color{Green}{\boldsymbol{{\phi_{\mathrm{eog}}}}}}$ and cardiac ${\color{red}{\boldsymbol{{\phi_{\mathrm{ecg}}}}}}$) are collected. (B) Independent component analysis (ICA) decomposes $\bm{X}_{\mathrm{raw}}$ into independent components, which are correlated with references to identify artifact components. (C) Artifact components are projected back to sensor space, stochastically scaled using their empirical amplitudes, and remixed to generate augmented samples.}
    \label{fig: summaryfig}
    \vspace{-1em}
\end{figure}

We address this gap by introducing \textit{physiological noise augmentation} (PNA), a data augmentation framework that isolates artifact-related independent components (e.g., ocular and cardiac activity) and re-injects remixed variants to generate physiologically realistic, label-preserving training examples. By exposing the decoder to a wider range of artifact realizations, PNA encourages invariance to nuisance structure and promotes reliance on task-relevant signals. Combined with multi-trial averaging to reduce residual variability, this approach reduces the number of repetitions required to achieve strong decoding performance.

\paragraph{Contributions.} We introduce PNA, an augmentation framework for noninvasive brain-to-speech that promotes invariance to task-agnostic artifacts by remixing artifact-related independent components. PNA generates realistic sensor-space perturbations by scaling artifacts according to the empirical distribution of energy ratios. We provide theoretical support for PNA by showing that, under multi-trial averaging, the method approximates anisotropic Jacobian regularization that penalizes the decoder's sensitivity to directions aligned with tracked artifacts.  Empirically, PNA substantially improves performance on the MegNIST imagined speech dataset, increasing decoding accuracy by $4.7\%$ (absolute) over real-data baselines and reaching $77.6\%$ decoding accuracy with EEGNet.

\section{Related Work}
\label{app: related_work}

We explore three threads of related literature that define the technical setting of this paper: non-invasive neural decoding, artifact handling in M/EEG, and data augmentation for neural time series.

\textbf{Non-Invasive Speech Decoding.} Recent progress into non-invasive decoding has focused predominantly on the task of \textit{perceived} speech, and has been driven largely by MEG data scaling and self-supervised learning \cite{jiang2024large}. \cite{defossez2023nature} apply contrastive learning on aligned embeddings of audio and neural responses, \cite{dascoli2025nature} scale across subjects, devices, and languages for transferable word-level retrieval, and \cite{jayalath2025unlocking} integrate language-model rescoring to push toward open-vocabulary brain-to-text. These methods exploit a structural advantage of perceived speech: the acoustic stimulus provides a higher SNR, largely time-locked target that anchors learning \citep{anumanchipalli2019speech}. In contrast, imagined speech, typically provides weaker SNRs and less predictable temporal alignment \citep{proix2022imagined}. We study this harder setting via single-subject MEG classification of imagined digits and address the noise problem directly through the artifact-aware augmentation.

\textbf{Artifact Handling in M/EEG.} Physiological artifacts that appear in MEG and EEG signals, such as ocular, muscular or cardiac activity, are typically treated as nuisance signals to be suppressed. The dominant pipeline for artifact suppression identifies and removes physiological artifacts via independent component analysis (ICA) \citep{bellsejnowski1995infomaxica, jung2000removing, vigario2002independent}, with recent work automating component selection via reference-channel correlation, topographies, or learned classifiers \citep{pion2019iclabel, winkler2011automatic}. PNA inverts this convention: rather than discarding artifact components, we use the same ICA decomposition to generate physiologically realistic training perturbations, exposing the decoder to the artifact distribution it will encounter at test time.

\textbf{Data Augmentation for Neural Decoding.} Augmenting clean inputs with realistic nuisances has a long history in automatic speech recognition (ASR), where noise corpora \citep{snyder2015musan} and reverberation \citep{ko2017study} are routinely added to clean speech to improve test-time robustness. Neural signal augmentations to date have imported the input-level transformations (e.g. time-frequency warping \cite{wang2025tft}, SpecAugment-style masking \cite{dezuazo2025megconformer,park2019specaugment}),  without the corresponding ASR principle of drawing perturbations from the empirical noise distribution of the target domain. PNA closes this gap: by sampling artifact realizations directly from auxiliary EOG/ECG references and the decomposed signal, it produces augmentations that match the empirical artifact statistics of the recording session.

\section{Method}
\label{sec:method}

We model brain-to-speech decoding as a multiclass classification problem over a fixed vocabulary, $\mathcal{V}$, of size $V$. Let $\mathbf{X} = [\mathbf{x}_{1},\ldots,\mathbf{x}_{T}] \in \mathbb{R}^{C\times T}$ represent a spatiotemporal neural recording, where $C$ is the number of channels, $T$ is the number of time samples per word trial, and $\mathbf{x}_{t}\in\mathbb{R}^{C}$ denotes the sensor measurements at time $t$. Our goal is to learn a decoder $f_{\theta}: \mathbb{R}^{C \times T} \rightarrow \Delta^{V-1}$ that maps the input signal to a vector of word probabilities in the $(V-1)$-simplex, such that the $v$-th element, $f_{\theta}(\mathbf{X})_v$, denotes the predicted probability of word $v \in \mathcal{V}$.

In practice, the observed recording is an additive mixture $\mathbf{X} = \mathbf{X}_{\text{task}} + \mathbf{X}_{\text{artifact}}$, where $\mathbf{X}_{\text{task}}$ represents task-relevant neural activity and $\mathbf{X}_{\text{artifact}} \in \mathcal{A}$ represents task-irrelevant physiological artifacts (e.g. cardiac or ocular activity), residing in an artifact subspace, $\mathcal{A} \subset \mathbb{R}^{C\times T}$. To ensure the decoder's predictions are driven exclusively by neural intent, we seek a model that is invariant to the artifact subspace, satisfying:
\begin{align}
    f_{\theta}(\mathbf{X}_{\mathrm{task}} + \mathbf{X}_{\mathrm{artifact}}) = f_{\theta}(\mathbf{X}_{\mathrm{task}}),  \quad \quad \forall \ \mathbf{X}_{\mathrm{task}} \in \mathbb{R}^{C\times T}, \ \mathbf{X}_{\mathrm{artifact}} \in \mathcal{A}
    \label{eqn: decoder_invariance}
\end{align}
to ensure that the predicted distribution over $\mathcal{V}$ depends only on task-related neural signals.

\subsection{Physiological Noise Augmentation}
\label{sec:physiological-noise-augmentation}

PNA seeks to enforce Equation \eqref{eqn: decoder_invariance} in three steps (Figure \ref{fig: summaryfig}): (i) recording nuisance reference channels during data acquisition, (ii) identifying and removing artifact-related ICA components, and (iii) re-injecting scaled artifact projections into the cleaned data to generate physiologically plausible synthetic samples.

\textbf{(i) Measuring artifacts.} During recording experiments spanning $T_{\mathrm{total}}$ time steps, we collect reference waveforms for task-agnostic artifacts.\footnote{Recording experiments are generally divided into several sessions, with independent component analysis repeated for each session to account for changes in sensor placement.} In MegNIST data used for our experiments, electrooculography (EOG) and electrocardiography (ECG) are recorded to capture ocular and cardiac activity, respectively (Figure \ref{fig: summaryfig}A). Let $\bm{\phi}_p \in \mathbb{R}^{T_{\mathrm{total}}}$ denote the reference time series for artifact $p \in \mathcal{P}$.

\textbf{(ii) Removing artifact-correlated sources.} PNA assumes that sensor recordings may be separated into statistically independent source components using \textit{independent component analysis} (ICA) \citep{bellsejnowski1995infomaxica}, which we overview in Appendix \ref{app: ICA}. We fit FastICA \cite{hyvarinen1999fastica} on recordings using MNE-Python \cite{gramfort2013meg}, yielding source estimates $\widehat{\mathbf{S}}$.

Artifact-related components are identified by their correlation with reference channels. For each tracked artifact type $p \in \mathcal{P}$, we define
\begin{align*}
\mathcal{S}_{p} = \left\{ i \in \{1, \ldots, N\} \;:\; \left|\rho\!\left(\hat{\mathbf{s}}_{i},\, \bm{\phi}_{p}\right)\right| \ge \tau_{p} \right\}
\end{align*}
as the set of associated nuisance component indices. Here, $\hat{\mathbf{s}}_{i} \in \mathbb{R}^{T}$ is the time-series of the $i$-th ICA component (the $i$-th row of $\widehat{\mathbf{S}}$), $\bm{\phi}_{p}$ denotes the reference time-series for artifact $p$, and $\rho(\cdot, \cdot)$ is Pearson correlation computed over $T_{\mathrm{total}}$ time samples. $\tau_{p}$ are selected correlation thresholds.

We project each artifact's components back to sensor space as $\mathbf{X}_{p} = \mathbf{A}_{:,\mathcal{S}_p}\widehat{\mathbf{S}}_{\mathcal{S}_p,:}$, and recover the cleaned recording by subtraction, $\mathbf{X}_{\mathrm{clean}} = \mathbf{X}_{\mathrm{raw}} - \sum_{p \in \mathcal{P}} \mathbf{X}_p$, following the standard ICA cleaning convention used in MNE-Python. Figure \ref{fig:ICA-art} in Appendix \ref{apx: data_vis} shows a topographical and waveform view of extracted ICA components using EOG (ocular) and ECG (cardiac) reference signals, alongside ICA component correlations.

\begin{figure*}[t]
    \centering
    \includegraphics[width=\linewidth]{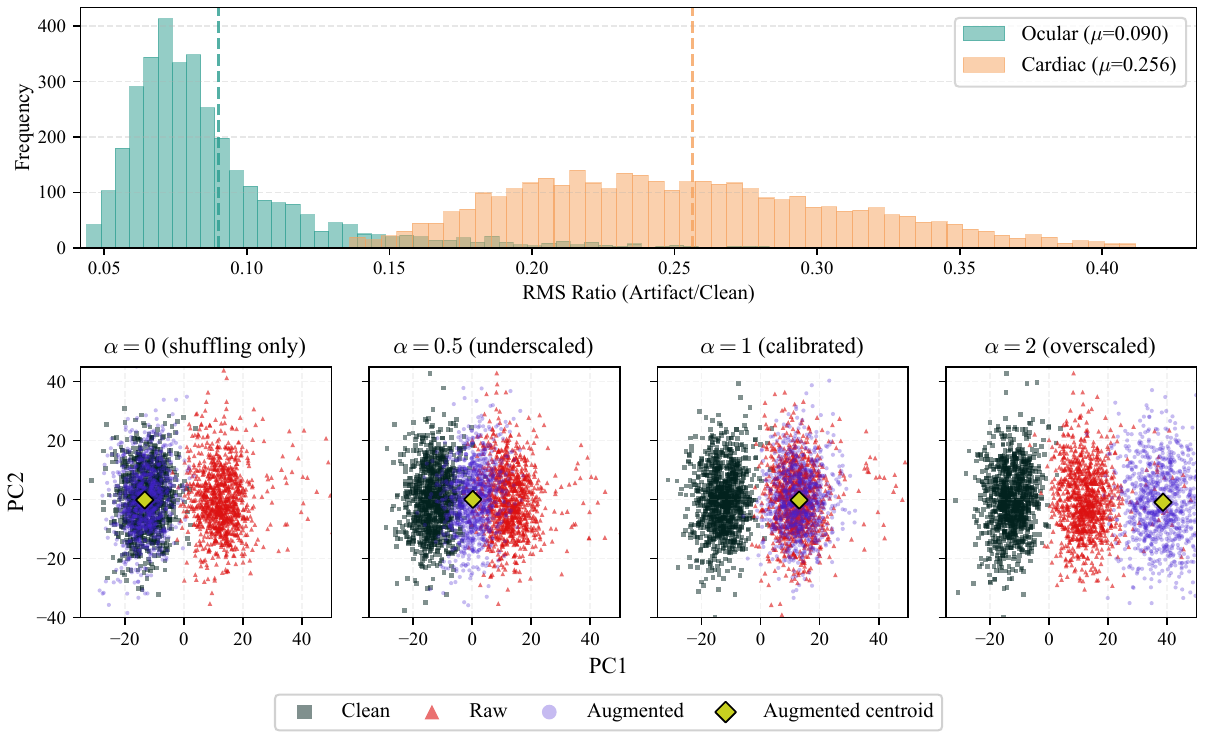}
    \caption{\looseness=-1 \textbf{Stochastic amplitude sampling calibrates PNA-augmented MEG embeddings to the distribution of raw data.} \textbf{(Top)} Empirical artifact-to-clean amplitude ratio distributions (zeros removed). \textbf{(Bottom)} PCA fit to 15-trial-averaged raw (red) and clean (black) embeddings of imagined digits; augmented samples (blue, gold centroid) are projected at varying scalings $\alpha$. Augmented embeddings interpolate between clusters for $\alpha \in [0,1]$, match Eq.~\eqref{eq:ecg} at $\alpha=1$, and shift off the biological support for $\alpha>1$. Pointwise misalignment at $\alpha=0$ reflects random indexing of artifact components.}
    \label{fig:pca-tsne}
\end{figure*}

\textbf{(iii) Augmenting with physiological noise.} We generate augmented trials by remixing scaled artifact reconstructions back into the cleaned signal,
\begin{align}
\mathbf{X}_{\mathrm{aug}} = \mathbf{X}_{\mathrm{clean}} + \sum_{p \in \mathcal{P}} c_p \mathbf{X}_p,
\label{eq: aug_eqn_general}
\end{align}
where each $c_p$ is a stochastically sampled scaling coefficient chosen to span the empirical range of signal-to-artifact ratios.

Let $\|\cdot\|_F$ denote the Frobenius norm over sensors (of the same type) and time, and let $\mathbf{X}_{p,d}$ and $\mathbf{X}_{\mathrm{clean},d}$ denote the reconstructions for trial $d \in \{1, \ldots, D\}$. We define the empirical artifact-to-clean amplitude ratios
\begin{align*}
r_{p,d} = \frac{\|\mathbf{X}_{p,d}\|_F}{\|\mathbf{X}_{\mathrm{clean},d}\|_F + \varepsilon},
\end{align*}
with $\varepsilon > 0$, a small numerical-stability constant. Figure \ref{fig:pca-tsne} (Top) shows the resulting distributions for EOG and ECG; ECG components carry roughly twice the relative energy of EOG.

During each forward pass, we sample a donor trial index $j \sim \mathrm{Uniform}\{1, \ldots, D\}$ and set
\begin{align}
c_{p} = \alpha_p \cdot r_{p,j} \cdot \frac{\|\mathbf{X}_{\mathrm{clean}}\|_F}{\|\mathbf{X}_{p,j}\|_F + \varepsilon}, \qquad \forall\, p \in \mathcal{P},
\label{eq:ecg}
\end{align}
which rescales the donor artifact $\mathbf{X}_{p,j}$ so its amplitude relative to $\mathbf{X}_{\mathrm{clean}}$ matches the sampled ratio $r_{p,j}$.For generality, a fixed per-artifact scaling parameter, $\alpha_p$, enables fine-grained control over the expected artifact magnitude. In our experiments, however, we set $\alpha_p = 1$ for all $p$ to maximize the fidelity of the augmented data, as illustrated in Figure \ref{fig:pca-tsne} (bottom). Augmented samples therefore span the empirical artifact-to-signal distribution, encouraging the decoder toward invariance under variable contamination at test time. We provide pseudocode in Algorithm \ref{alg:physio_aug_generalized}.

\begin{algorithm}[h]
\caption{Physiological Noise Augmentation}
\label{alg:physio_aug_generalized}
\begin{algorithmic}
\Require Training recording $\mathbf{X}_{\text{raw}}$ partitioned into $D$ trials $\{\mathbf{X}_d\}_{d=1}^D$; references $\{\bm{\phi}_p\}_{p \in \mathcal{P}}$; thresholds $\{\tau_p\}$; per-artifact scales $\{\alpha_p\}$; stability constant $\varepsilon$
\vspace{0.6em}
\State \textbf{// Stage 1: Offline preprocessing (run once)}
\State $\mathbf{W}, \mathbf{A} \gets \text{FastICA}(\mathbf{X}_{\text{raw}})$;\; $\widehat{\mathbf{S}} \gets \mathbf{W}\mathbf{X}_{\text{raw}}$ \Comment{ICA decomposition}
\For{$p \in \mathcal{P}$} \Comment{Assume $\mathcal{S}_p$ disjoint across $p$}
    \State $\mathcal{S}_p \gets \{ i : |\rho(\hat{\mathbf{s}}_i, \bm{\phi}_p)| \ge \tau_p \}$ \Comment{Artifact-correlated components}
    \State $\mathbf{X}_p \gets \mathbf{A}_{:,\mathcal{S}_p}\widehat{\mathbf{S}}_{\mathcal{S}_p,:}$ \Comment{Project artifact $p$ to sensor space}
\EndFor
\State $\mathbf{X}_{\text{clean}} \gets \mathbf{X}_{\text{raw}} - \sum_{p \in \mathcal{P}} \mathbf{X}_p$ \Comment{Cleaned recordings}
\For{$d = 1 \dots D,\; p \in \mathcal{P}$}
    \State $r_{p,d} \gets \|\mathbf{X}_{p,d}\|_F \,/\, (\|\mathbf{X}_{\text{clean},d}\|_F + \varepsilon)$ \Comment{Empirical artifact-to-clean ratio}
\EndFor
\vspace{0.6em}
\State \textbf{// Stage 2: Online augmentation (called per forward pass)}
\Function{Augment}{$\mathbf{X}_{\text{clean},d}$}
    \State Sample $j \sim \text{Uniform}\{1, \dots, D\}$ \Comment{Random artifact donor trial}
    \For{$p \in \mathcal{P}$}
        \State $c_p \gets \alpha_p \cdot r_{p,j} \cdot \|\mathbf{X}_{\text{clean},d}\|_F \,/\, (\|\mathbf{X}_{p,j}\|_F + \varepsilon)$ \Comment{Calibrate scaling}
    \EndFor
    \State \Return $\mathbf{X}_{\text{clean},d} + \sum_{p \in \mathcal{P}} c_p \mathbf{X}_{p,j}$
\EndFunction
\end{algorithmic}
\end{algorithm}

\looseness=-1 Thus far, we have considered onlythe effect of \textit{tracked} artifacts on decoding. Optionally, after applying PNA, we further apply \textit{multi-trial averaging} to suppress residual untracked artifacts that are not phase-locked to internal speech onset. Multi-trial averaging has long been used to improve SNR in noninvasive brain-to-text systems \cite{farwell1988talking}; a brief overview is provided in Appendix \ref{app: averaging}. 

Figure \ref{fig:tsne-class-clustering} presents t-SNE embeddings \cite{van2008visualizing} of MegNIST data before and after averaging with resampling. Without averaging, t-SNE fails to identify well-separated local clusters corresponding to imagined digit classes. In contrast, clear class-specific clusters emerge after 15-trial averaging. We repeat this analysis in Figure \ref{fig:UMAP_classes} of Appendix \ref{apx: data_vis} using UMAP \cite{mcinnes2018umap}, which provides a more faithful representation of global inter-class structure.

\begin{figure}[t]
    \centering
    \includegraphics[width=\linewidth]{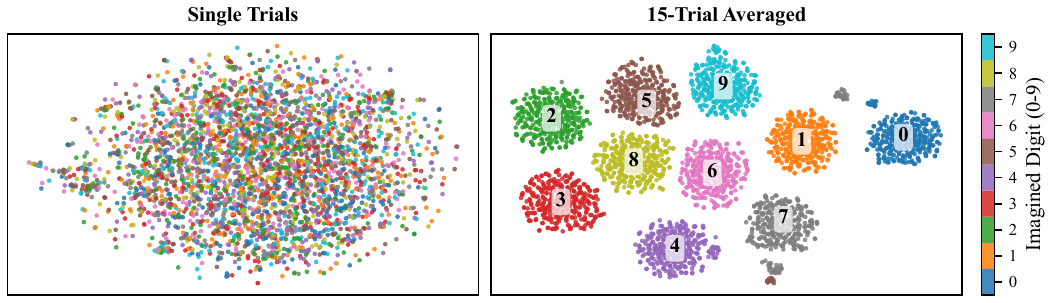}
    \caption{\textbf{t-SNE embeddings of single-trial (left) and 15-trial averaged (right) intra-patient MEG recordings of imagined digits (0--9)}. Averaging is performed with resampling to preserve dataset cardinality; t-SNE uses $\text{perplexity} = 30$.}
    \label{fig:tsne-class-clustering}
\end{figure}   

\subsection{Theoretical Motivation}
\label{sec:theoretical-motivation}

Combining PNA's artifact injection (Section 3.1 (iii)) with $K$-trial averaging yields training inputs of the form $\bar{\mathbf{x}} + \delta$, where $\bar{\mathbf{x}}$ is the trial-averaged signal and $\delta = \sum_{p \in \mathcal{P}} c_p \mathbf{X}_p$ is the injected artifact mixture. By construction, $\delta$ has zero mean and covariance $\Sigma_\delta = (\alpha^2/K)\,\Sigma_n$, where $\Sigma_n$ is the empirical artifact covariance estimated from the sampled $\{c_p \mathbf{X}_p\}$ ensemble and $\alpha$ is the augmentation scaling parameter, whose effect is visualized with $K=15$ in Figure~\ref{fig:pca-tsne}. We now show that, in expectation and to leading order, this procedure is equivalent to training on clean inputs with an added penalty that suppresses decoder sensitivity along directions of high artifact variance. 

\vspace{0.5em}

\begin{proposition}[PNA as anisotropic regularization]
\label{thm: general_loss}
Let $z(\bar{\mathbf{x}}; \theta) \in \mathbb{R}^V$ be the logits of a $C^3$ decoder, $\ell(z, y)$ a $C^3$ loss criterion with targets, $y$, and define the loss gradient $g = \nabla_z \ell$, loss Hessian $H_\ell = \nabla_z^2 \ell$, logit Jacobian $J_z = \partial_{\bar{\mathbf{x}}} z$, and logit Hessian $H_{z_r} = \nabla_{\bar{\mathbf{x}}}^2 z_r$ ($r = 1, \dots, V$), all evaluated at the clean $K$-trial average input $\bar{\mathbf{x}}$. Suppose the augmented input $\mathbf{x}_{\mathrm{aug}}=\bar{\mathbf{x}} + \delta$ has $\mathbb{E}[\delta \mid \bar{\mathbf{x}}, y] = 0$ and $\mathrm{Cov}(\delta) = (\alpha^2/K)\Sigma_n$. Then
\begin{align}
    \mathbb{E}_{\delta}[\ell_{\mathrm{aug}}(z, y)]
    \;=\;
    \ell_{\mathrm{clean}}(z, y)
    \;+\;
    \frac{\alpha^2}{2K} \!\left[\,
        \mathrm{Tr}\!\left(H_\ell J_z \Sigma_n J_z^\top\right)
        +
        \sum_{r=1}^V g_r\, \mathrm{Tr}(H_{z_r} \Sigma_n)
    \,\right]
    \;+\;
    o\!\left(\tfrac{\alpha^2}{K}\right).
\end{align}
If, in addition, either (a) $\mathrm{Tr}(H_{z_r}\Sigma_n) \approx 0$ for each $r$ (the decoder is locally linear in artifact-covariance directions), or (b) $g \approx 0$ at $\bar{\mathbf{x}}$ (the model predicts near-correctly at this data point), the second-order term simplifies to a loss-curvature-weighted Jacobian penalty:
\begin{align}
    \mathbb{E}_{\delta}[\ell_{\mathrm{aug}}(z,y)]
    \;\approx\;
    \ell_{\mathrm{clean}}(z,y)
    \;+\;
    \frac{\alpha^2}{2K}\, \mathrm{Tr}\!\left(H_\ell J_z \Sigma_n J_z^\top\right).
\end{align}
\end{proposition}

\textit{Proof.} See Appendix \ref{app:general_loss}.

\vspace{1em}

\begin{corollary}[Anisotropic Jacobian Regularization under Squared Error Loss]
Let $\ell(z, y) = \|y - z\|_2^2$ be the squared error loss. Under the assumptions of Proposition \ref{thm: general_loss}, and provided either (a) $\mathrm{Tr}(H_{z_r} \Sigma_n) \approx 0$ for each $r$, or (b) $g = -2(y - z) \approx 0$ at $\bar{\mathbf{x}}$, applying this loss to augmented data averaged over $K$ trials is equivalent to applying it to clean data, with an anisotropic Jacobian penalty:
\[
\mathbb{E_{\delta}}[\ell_{\mathrm{aug}}(z,y)]
\approx 
\ell_{\mathrm{clean}}(z,y) 
+ 
\frac{\alpha^2}{K}\mathrm{Tr}\left( J_z \Sigma_n J_z^\top \right),
\]
where $\Sigma_\delta = \frac{\alpha^2}{K} \Sigma_n$ is the covariance of the injected artifact noise.
\end{corollary}

\textit{Proof.} See Appendix \ref{thm:corr1}.

\vspace{1em}

\begin{corollary}[PNA Regularization under Cross-Entropy Loss]
\label{thm: ce_loss}
Let $\ell(z, y) = -\sum_{i=1}^V y_i \log(\sigma(z)_i)$ be the softmax cross-entropy loss, where $\sigma(z)_i = e^{z_i} / \sum_j e^{z_j}$ denotes the predicted probability for class $i$. Under the assumptions of Proposition \ref{thm: general_loss}, and provided either (a) $\mathrm{Tr}(H_{z_r} \Sigma_n) \approx 0$ for each $r$, or (b) $g = p - y \approx 0$ at $\bar{\mathbf{x}}$, the augmented loss reduces to minimizing the clean loss plus a Fisher-weighted anisotropic Jacobian penalty:
\[
\mathbb{E_{\delta}}[\ell_{\mathrm{aug}}(z, y)]
\approx 
\ell_{\mathrm{clean}}(z, y) 
+ 
\frac{\alpha^2}{2K} \mathrm{Tr}\left( \left[ \mathrm{Diag}(p) - p p^\top \right] J_z \Sigma_{n} J_z^\top \right),
\]
where $p = \sigma(z)$ is the vector of predicted probabilities and $\Sigma_\delta = \frac{\alpha^2}{K} \Sigma_n$ is the covariance of the injected artifact noise.
\end{corollary}

\textit{Proof.} See Appendix \ref{thm:proof_CE}.

Proposition 1 and its corollaries are related to the classical equivalence between Gaussian noise and Tikhonov regularization \cite{bishop1995training}. Notably, in order to ensure that the resulting regularizer is non-negative, the loss criteria must be convex with respect to the logits (a condition satisfied by both squared error and cross-entropy). The simplification of the theorem to a pure Jacobian penalty requires the second-order logit Hessian terms to vanish, which occurs if either (a) the decoder is locally linear in the directions of artifact noise, or (b) the loss gradient $g$ is approximately zero. For a classifier trained under cross-entropy, condition (b) corresponds to the regime where predictions are well-aligned with targets ($p \approx y$). Consequently, this Jacobian approximation becomes increasingly precise as the model converges, characterizing the effective regularization landscape during the later stages of training.

\subsection{Preprocessing \& Implementation Details}

In our pipeline, we perform data augmentation before preprocessing so that augmentation preserves the original relative scales and variances of physiological signals. We follow the preprocessing strategy of \cite{defossez2023nature}, applying a robust scaler fitted only to the training data to avoid leakage. Similarly, ICA and augmentation are only applied to training data, and validation/test sets are recorded separately to prevent leakage. Full details are presented in Appendix~\ref{app: model_data_hyp}.

\section{Experiments}

\begin{figure*}[t]
    \centering
    \includegraphics[width=\linewidth]{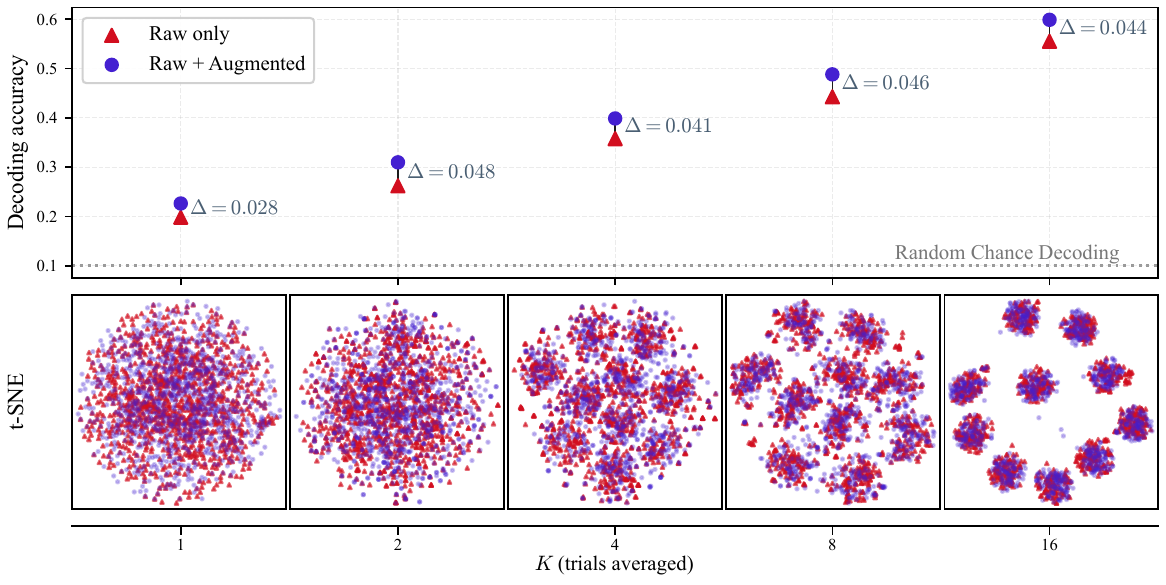}
    \caption{\textbf{MLP decoding accuracy uplift from PNA augmentation and associated t-SNE embeddings at various levels of averaging}. The top plot uses $10k$ training trials for raw-only runs (red) and an additional $10k$ augmented trials for raw + augmented runs (blue), both over 5 seeds. The bottom plots show t-SNE embeddings of raw (red triangles) and augmented (blue dots) data at each level of averaging.}
    \label{fig:k-sweep}
\end{figure*}

We evaluate PNA on the task of classifying imagined digits from MEG data, and compare to a set of common brain-to-speech augmentations. 

\textbf{Dataset.} We conduct experiments on MegNIST, a single-subject MEG dataset comprising 5 hours (12,000 trials) of class-balanced imagined digits (0–9). We choose MegNIST because (i) to our knowledge, it is the only publicly available MEG dataset for internal speech, and (ii) it includes EOG and ECG reference signals alongside MEG recordings, enabling PNA. Preprocessing and data splits are detailed in Table \ref{tab:preprocessing} of Appendix \ref{app: model_data_hyp}.

\textbf{Models.} For each augmentation setting, we train two models: a multilayer perceptron (MLP) and EEGNet \citep{lawhern2018eegnet}. This choice enables a direct comparison between a simple architecture with flattened inputs and a modern convolutional architecture operating on 2D (unflattened) inputs. We optimize the model hyperparameters with grid search, and report final values in Appendix~\ref{app: model_data_hyp}.

\subsection{Decoding Performance and the Effect of Averaging}

\begin{table}[t]
    \begin{center}
        \begin{small}
            \begin{sc}
                \begin{tabular}{lcc}
                    \toprule
                    \textit{Data Strategy} & MLP & EEGNet \\
                    \midrule
                    \textit{Clean Only} & $0.255$ {\scriptsize $\pm$ $0.008$} & $0.321$ {\scriptsize $\pm$ $0.004$} \\
                    \quad + Averaging & $\textbf{0.581}$ {\scriptsize $\pm$ $0.015$} & $0.678$ {\scriptsize $\pm$ $0.022$} \\
                    \addlinespace[6pt]
                    \textit{Raw Only} ($p=0$) & $0.255$ {\scriptsize $\pm$ $0.008$} & $0.337$ {\scriptsize $\pm$ $0.005$} \\
                    \quad + Averaging & $0.542$ {\scriptsize $\pm$ $0.011$} & $0.730$ {\scriptsize $\pm$ $0.014$} \\
                    \addlinespace[6pt] 
                    \textit{Raw + PNA}  ($p=0.5$) & $0.262$ {\scriptsize $\pm$ $0.006$} & $0.332$ {\scriptsize $\pm$ $0.011$} \\
                    \quad + Averaging & $0.574$ {\scriptsize $\pm$ $0.020$} & $\textbf{0.763}$ {\scriptsize $\pm$ $0.011$} \\
                    \addlinespace[6pt]
                    \textit{PNA Only} ($p=1$) & $0.267$ {\scriptsize $\pm$ $0.005$} & $0.319$ {\scriptsize $\pm$ $0.017$} \\
                    \quad + Averaging & $0.504$ {\scriptsize $\pm$ $0.030$} & $0.626$ {\scriptsize $\pm$ $0.015$} \\
                    \bottomrule
                \end{tabular}
            \end{sc}
        \end{small}
    \end{center}
    \caption{\textbf{Ablation of Data Augmentation Strategies.} \textit{Raw Only}, \textit{Clean Only}, and \textit{PNA Only} each use $10k$ training trials. \textit{Raw + PNA} applies augmentation online during training with probability $p=0.5$. PNA is applied before 10-trial averaging. Random chance decoding yields an expected accuracy of $0.1$. Best results over 5 seeds are shown in \textbf{bold}; standard errors are reported in \scriptsize{scriptsize}.}
    \label{tab:data_aug_results}
\end{table}

\begin{figure}[t]
    \centering
    \includegraphics[width=\linewidth]{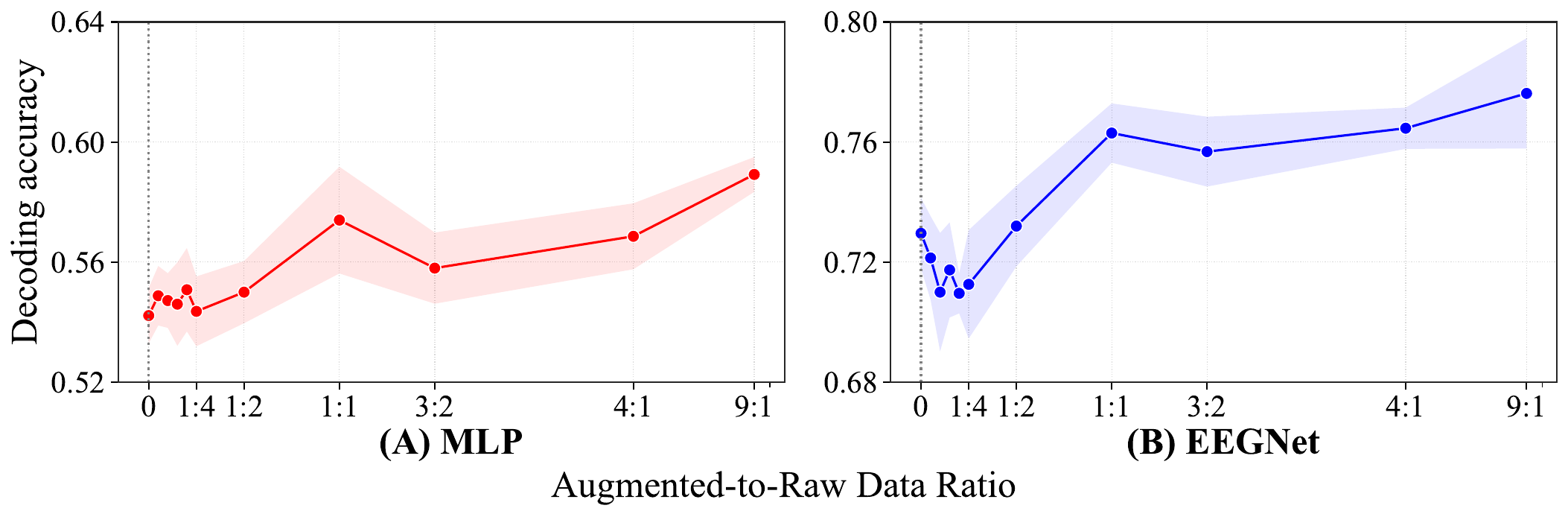}
    \caption{\textbf{Decoding accuracy scales with PNA samples.} Models are trained on a fixed budget of $10k$ raw trials; the $x$-axis represents the ratio of augmented-to-raw data, ranging from 0 (baseline, vertical dotted line) to $90k$ ($9:1$ ratio). Shaded regions denote $95\%$ confidence intervals ($n=5$ independent seeds); all runs use 10-trial averaging. }
    \label{fig:num_aug_results}
\end{figure}

Table \ref{tab:data_aug_results} shows an ablation of PNA, using various combinations of raw data, clean data, augmented data, and 10-trial averaging.  The framework does not overcome the SNR bottleneck in the single-trial, however its impact is substantially amplified when combined with multi-trial averaging. With 10-trial averaging, accuracies increase from $54.2\% \pm 1.1\%$ to $57.4\% \pm 2.0\%$ for MLP and from $73.0\% \pm 1.4\%$ to $76.3\% \pm 1.1\%$ for EEGNet when augmented samples are included. This result suggests that augmentation is most effective once averaging has mitigated extreme low-SNR conditions, allowing models to better exploit invariances introduced by PNA. Full results across models and training configurations are provided in Appendix~\ref{app: full_pna_results}. We use a 1:1 raw-to-augmented data ratio for simplicity and as supported by Figure \ref{fig:num_aug_results}.

\subsection{Comparison to Augmentation Baselines}

Table \ref{tab:augmentation_results_old} compares the effects of several augmentation baselines, with and without 10-trial averaging. We find that PNA, when added to existing raw data, is able to outperform training on raw data with MLP, and on 10-trial averaged data with EEGNet, the latter of which achieves the highest decoding accuracy among all baselines, $76.3\% \pm 1.1\%$. Notably, Figure \ref{fig:k-sweep} provides empirical evidence that the enhancement under PNA remains relatively consistent beyond 2 trials of averaging.

\begin{table*}[h!t]
\centering
\begin{tabular}{l c c c c}
\toprule
 & \multicolumn{2}{c}{\textsc{MLP}} & \multicolumn{2}{c}{\textsc{EEGNet}} \\
\cmidrule(lr){2-3} \cmidrule(lr){4-5}
\textit{Augmentation Type} & Single-Trial & 10-Trial & Single-Trial & 10-Trial \\
\midrule
\textit{None (Raw Data)}     & $0.255$ {\scriptsize $\pm 0.008$} & $0.542$ {\scriptsize $\pm 0.011$} & $0.337$ {\scriptsize $\pm 0.005$} & $0.730$ {\scriptsize $\pm 0.014$} \\
\addlinespace

{Raw + White Noise}        & $0.261$ {\scriptsize $\pm 0.007$} & $0.578$ {\scriptsize $\pm 0.011$} & $0.329$ {\scriptsize $\pm 0.006$} & $0.747$ {\scriptsize $\pm 0.012$} \\

{Raw + Smooth Time Mask}   & $0.256$ {\scriptsize $\pm 0.008$} & $0.573$ {\scriptsize $\pm 0.015$} & $0.317$ {\scriptsize $\pm 0.009$} & $0.713$ {\scriptsize $\pm 0.023$} \\

{Raw + Frequency Shift}    & $\textbf{0.270}$ {\scriptsize $\pm 0.005$} & $\textbf{0.596}$ {\scriptsize $\pm 0.022$} & $0.321$ {\scriptsize $\pm 0.010$} & $0.739$ {\scriptsize $\pm 0.016$} \\

{Raw + Temporal Shift}     & $0.265$ {\scriptsize $\pm 0.004$} & $0.584$ {\scriptsize $\pm 0.019$} & $0.328$ {\scriptsize $\pm 0.006$} & $0.736$ {\scriptsize $\pm 0.016$} \\

{Raw + Amplitude Scaling}  & $0.260$ {\scriptsize $\pm 0.010$} & $0.555$ {\scriptsize $\pm 0.017$} & $\textbf{0.341}$ {\scriptsize $\pm 0.009$} & $0.735$ {\scriptsize $\pm 0.021$} \\
\addlinespace

{Raw + PNA}                & $0.262$ {\scriptsize $\pm 0.006$} & $0.574$ {\scriptsize $\pm 0.020$} & $0.332$ {\scriptsize $\pm 0.011$} & $\textbf{0.763}$ {\scriptsize $\pm 0.011$} \\
\bottomrule
\end{tabular}
\caption{\textbf{Comparison against existing brain-to-speech augmentations.} All augmentation methods use $10k$ raw trials with augmentation applied online with probability $0.5$. \textit{None} uses $10k$ raw trials without augmentation. Averaging is performed after PNA and before each baseline augmentation. Best results are shown in \textbf{bold}; standard errors are reported in \scriptsize{scriptsize}.}
\label{tab:augmentation_results_old}
\end{table*}

\section{Discussion}
\label{sec:discussion}

We present Physiological Noise Augmentation (PNA), an ICA-based framework that enforces invariance to physiological artifacts in non-invasive brain-to-speech decoding. By remixing artifact components to generate diverse, label-preserving samples, PNA improves decoding accuracy by $4.7\%$ over real-data training alone for MEG-based imagined digit classification with EEGNet. We further demonstrate that PNA and multi-trial averaging are complementary: PNA reduces sensitivity to tracked nuisances, while averaging suppresses residual task-agnostic variability, together approximating Jacobian regularization. 

\textbf{Limitations and Future Work.} PNA currently requires artifact reference signals to be recorded during data collection. Future work will investigate classification methods to identify nuisance artifact components from ICA on MEG, eliminating the need for reference channels and substantially expanding the range of compatible datasets. In addition, MegNIST is limited to a single subject performing a constrained task. Extending PNA to multi-subject datasets and broader internal speech vocabularies will clarify the transferability of artifacts.

More broadly, PNA is designed to benefit from access to a wider range of task-agnostic physiological artifacts. Beyond EOG and ECG, future work should incorporate additional high-amplitude reference modalities, such as EMG, to maximally teach the decoder to become invariant to task-irrelevant artifacts.

\textbf{Impact and Implications.}

The development of PNA represents a methodological shift from heuristic input-level perturbations toward a physiologically-grounded regularization framework. PNA provides a principled mechanism for injecting domain-specific inductive biases into neural decoders by explicitly drawing from reference channels to isolate structured nuisances. Our theoretical and empirical results demonstrate that the associated anisotropic Jacobian penalty substantially reduces the ``repetition burden'' required for high-accuracy decoding, directly addressing the core latency--accuracy trade-off that has long hindered the clinical viability of non-invasive BCIs. More broadly, PNA establishes a scalable template for artifact-aware machine learning in any neural recording modality where reference-tracked noise typically confounds task-relevant signals.

\bibliographystyle{plainnat} 
\bibliography{references}

\newpage
\appendix

\section{Additional Background}

\subsection{Independent Component Analysis}
\label{app: ICA}

\textit{Independent component analysis} (ICA) is modeled on the assumption that a matrix of spatiotemporal sensor recordings, $\mathbf{X}$, arises from instantaneous linear mixing of $N$ independent sources,
$\mathbf{S} = [\mathbf{s}_{1},\ldots,\mathbf{s}_{T}] \in\mathbb{R}^{N\times T}$ according to
\begin{align*}
\mathbf{x}_{t} &= \mathbf{A}\mathbf{s}_{t}, \qquad \text{for } t=1,\ldots,T,
\end{align*}
or equivalently,
\begin{align}
    \label{eq:mixing}
    \mathbf{X} &= \mathbf{A}\mathbf{S},
\end{align}
where $\mathbf{A}\in\mathbb{R}^{C\times N}$ is an unknown mixing matrix. \textit{Independent component analysis} (ICA) seeks to recover latent independent sources (rows of $\mathbf{S}$), by estimating an unmixing matrix $\mathbf{W}\in\mathbb{R}^{N\times C}$ which yields component estimates,
\begin{align}
\widehat{\mathbf{S}}=\mathbf{W}\mathbf{X}, \label{eq:unmixing}
\end{align}
such that the rows of $\widehat{\mathbf{S}}$ are as statistically independent as possible \cite{comon1994independent}. A canonical objective is to pick the optimal unmixing matrix, $\mathbf{W}^*$, to minimize the mutual information of the component random variables:
\begin{align}
\mathbf{W}^\star \in \operatorname*{\arg\min}_{\mathbf{W}\in\mathbb{R}^{N\times C}}
I\!\left(\hat{\mathbf{s}}_{1,:},\ldots,\hat{\mathbf{s}}_{N,:}\right),
\label{eq:ica_mi}
\end{align}
where \(\hat{\mathbf{s}}_{i,:} \in \mathbb{R}^{1\times T}\) is the \(i\)-th row of $\widehat{\mathbf{S}}$.

Notably, the estimation of $\mathbf{W}$ was originally formulated as training a single-layer neural network for unsupervised learning, as in the Infomax ICA algorithm \citep{bellsejnowski1995infomaxica}. However, current implementations typically use the FastICA algorithm, which improves computational efficiency by using a fixed-point iteration to maximize a non-Gaussianity contrast as proxy for statistical independence \citep{hyvarinen1999fastica, hyvarinenOja2000ica}.\footnote{Most ICA algorithms also approximate the mixing matrix, $\mathbf{A}$, by taking an inverse or pseudoinverse of $\mathbf{W}$.}

\subsection{Multi-Trial Averaging and Resampling}
\label{app: averaging}

We define the $K$-trial average for a set of trials $\{ \mathbf{X}^{(k)} \}_{k=1}^K$ as the spatiotemporal mean:
\begin{equation}
    \overline{\mathbf{X}} = \frac{1}{K}\sum_{k=1}^K \mathbf{X}^{(k)}.
\end{equation}
In brain-to-speech decoding, $K$-trial averaging at inference requires the patient to imagine the target stimulus $K$ times. This approach assumes that each trial consists of a phase-locked neural signal plus independent stochastic noise; consequently, averaging preserves the coherent signal component while the noise variance is attenuated by a factor of $K$. For $K$ condition-matched trials with covariance $\boldsymbol{\Sigma}$, the sample average $\overline{\mathbf{X}}$ satisfies $\mathrm{Cov}(\overline{\mathbf{X}}) = \boldsymbol{\Sigma}/K$ under the assumption of independence, yielding a $\sqrt{K}$ scaling in Signal-to-Noise Ratio (SNR) \cite{dawson1954summation}.

During training, we consider two averaging strategies:
\begin{enumerate}
    \item \textbf{Averaging without resampling:} Trials are partitioned into disjoint groups of size $K$ and averaged. While simple, this reduces the effective training set size by a factor of $K$, limiting input variability.
    \item \textbf{Averaging with resampling:} Subsets of size $K$ are randomly drawn (without replacement within each subset) from the full training set. For a dataset of size $N$, there are $\binom{N}{K}$ possible subsets, a combinatorial explosion (when $\min(K, N-K)  >> 0$) that allows for massive data augmentation while preserving the original dataset's coverage.
\end{enumerate}

In this work, we mainly employ averaging with resampling. To ensure a fair comparison with single-trial models, we fix the number of generated samples to match the original training set cardinality $N$, unless otherwise stated. For robustness, however, we also also re-run certain main results to verify that outcomes also hold empirically under averaging without resampling (Figure \ref{fig:averaging_without_resampling_results} in Appendix \ref{app: full_pna_results}).

Crucially, averaging implicitly infuses task awareness into output data, as phase-locked portions of the signal remain relatively unaffected, while other components destructively interfere.\footnote{MegNIST and similar datasets prompt participants to imagine each label at a screen-cued moment, enabling precise temporal alignment across trials.} 

\section{Proofs of Theoretical Results}
\label{app: proofs}

\subsection{Proof of Proposition 1}
\label{app:general_loss}

\begin{proof}
The proof proceeds by applying a multivariate Taylor expansion to the logits with respect to the input noise, then expanding the loss with respect to the perturbed logits.

\textbf{Step 1: Expansion of the logits.} 
Let $\Delta z = z(\bar{\mathbf{x}} + \delta;\theta) - z(\bar{\mathbf{x}};\theta)$. Expanding each coordinate $z_r$ to second order around $\bar{\mathbf{x}}$:
\[
\Delta z_r = (\nabla_{\bar{\mathbf{x}}} z_r)^\top \delta + \tfrac{1}{2} \delta^\top H_{z_r} \delta + \mathcal{O}(\|\delta\|^3),
\]

\textbf{Step 2: Expansion of the loss.}
Expanding $\ell(z + \Delta z, y)$ around $z = z(\bar{\mathbf{x}};\theta)$:
\[
\ell(z + \Delta z, y) = \ell(z, y) + g^\top \Delta z + \tfrac{1}{2} \Delta z^\top H_\ell \Delta z + \mathcal{O}(\|\Delta z\|^3).
\]
Since $\|\Delta z\| = \mathcal{O}(\|\delta\|)$, the trailing term is $\mathcal{O}(\|\delta\|^3)$.

\textbf{Step 3: Substitution.}
Substituting the expression for $\Delta z$:
\begin{itemize}
    \item $g^\top \Delta z = g^\top J_z \delta + \tfrac{1}{2} \sum_{r=1}^V g_r\, \delta^\top H_{z_r}\delta + \mathcal{O}(\|\delta\|^3)$.
    \item $\tfrac{1}{2}\Delta z^\top H_\ell \Delta z = \tfrac{1}{2}\,\delta^\top J_z^\top H_\ell J_z \delta + \mathcal{O}(\|\delta\|^3)$, where the cross term $(J_z\delta)^\top H_\ell\,\mathbf{q}(\delta)$ contributes at $\mathcal{O}(\|\delta\|^3)$ and is absorbed into the remainder.
\end{itemize}
Collecting,
\[
\ell(z + \Delta z, y) = \ell(z, y) + g^\top J_z \delta + \tfrac{1}{2}\delta^\top J_z^\top H_\ell J_z \delta + \tfrac{1}{2}\sum_{r=1}^V g_r\, \delta^\top H_{z_r}\delta + \mathcal{O}(\|\delta\|^3).
\]

\textbf{Step 4: Expectation and final substitution.}
Taking expectation over $\delta$ given $(\bar{\mathbf{x}}, y)$, using $\mathbb{E}_\delta[\delta] = 0$ and $\mathbb{E}_\delta[\delta^\top A \delta] = \mathrm{Tr}(A \Sigma_\delta)$:
\[
\mathbb{E}_\delta[\ell_{\mathrm{aug}}(z, y)] = \ell_{\mathrm{clean}}(z, y) + \tfrac{1}{2}\mathrm{Tr}(J_z^\top H_\ell J_z \Sigma_\delta) + \tfrac{1}{2}\sum_{r=1}^V g_r\, \mathrm{Tr}(H_{z_r}\Sigma_\delta) + \mathbb{E}_\delta[\mathcal{O}(\|\delta\|^3)].
\]
For $\delta$ with $\mathrm{Cov}(\delta) = (\alpha^2/K)\Sigma_n$ and finite third absolute moment, $\mathbb{E}_\delta[\|\delta\|^3] = \mathcal{O}((\alpha^2/K)^{3/2})$, so the remainder is $o(\alpha^2/K)$ as $\alpha^2/K \to 0$. Using trace cyclicity, $\mathrm{Tr}(J_z^\top H_\ell J_z \Sigma_\delta) = \mathrm{Tr}(H_\ell J_z \Sigma_\delta J_z^\top)$, and substituting $\Sigma_\delta = (\alpha^2/K)\Sigma_n$:
\[
\mathbb{E}_\delta[\ell_{\mathrm{aug}}(z, y)] = \ell_{\mathrm{clean}}(z, y) + \frac{\alpha^2}{2K}\,\mathrm{Tr}(H_\ell J_z \Sigma_n J_z^\top) + \frac{\alpha^2}{2K}\sum_{r=1}^V g_r\, \mathrm{Tr}(H_{z_r}\Sigma_n) + o\!\left(\tfrac{\alpha^2}{K}\right).
\]
\end{proof}

\subsection{Proof of Corollary 1}
\label{thm:corr1}
\begin{proof}
Under regime (a) or (b) of Proposition~\ref{thm: general_loss},
\begin{align}
\mathbb{E}_\delta[\ell_{\mathrm{aug}}(z, y)]
\;\approx\;
\ell_{\mathrm{clean}}(z, y)
\;+\;
\frac{\alpha^2}{2K}\, \mathrm{Tr}\!\left( H_\ell J_z \Sigma_n J_z^\top \right).
\label{eqn:mse-prop1}
\end{align}
For the squared error loss $\ell(z, y) = \|y - z\|_2^2 = \sum_{i=1}^V (y_i - z_i)^2$, the loss Hessian is constant:
\[
H_\ell = \nabla_z^2 \ell(z, y) = 2 I_V,
\]
where $I_V$ is the $V \times V$ identity matrix. Substituting into the regularization term of Equation~\eqref{eqn:mse-prop1}:
\[
\frac{\alpha^2}{2K}\, \mathrm{Tr}(2 I_V\, J_z \Sigma_n J_z^\top)
\;=\;
\frac{\alpha^2}{K}\, \mathrm{Tr}(J_z \Sigma_n J_z^\top),
\]
yielding
\[
\mathbb{E}_\delta[\ell_{\mathrm{aug}}(z, y)]
\;\approx\;
\ell_{\mathrm{clean}}(z, y)
\;+\;
\frac{\alpha^2}{K}\, \mathrm{Tr}(J_z \Sigma_n J_z^\top),
\]
as claimed.
\end{proof}

\subsection{Proof of Corollary 2}
\label{thm:proof_CE}

\begin{proof}
Under regime (a) or (b) of Proposition~\ref{thm: general_loss},
\begin{align}
\mathbb{E}_\delta[\ell_{\mathrm{aug}}(z, y)]
\;\approx\;
\ell_{\mathrm{clean}}(z, y)
\;+\;
\frac{\alpha^2}{2K}\, \mathrm{Tr}\!\left( H_\ell J_z \Sigma_n J_z^\top \right).
\label{eqn:restate_prop_ce}
\end{align}
For the softmax cross-entropy loss $\ell(z, y) = -\sum_{i=1}^V y_i \log p_i$ with $p = \sigma(z)$, the gradient with respect to the $j$-th logit is $g_j = p_j - y_j$, so the Hessian elements are
\[
(H_\ell)_{ij}
\;=\;
\frac{\partial g_j}{\partial z_i}
\;=\;
\frac{\partial p_j}{\partial z_i}
\;=\;
\begin{cases}
p_j(1 - p_j) & \text{if } i = j,\\[2pt]
-p_i p_j & \text{if } i \neq j,
\end{cases}
\]
using the standard derivative of the softmax. In closed form,
\[
H_\ell
\;=\;
\nabla_z^2 \ell(z, y)
\;=\;
\mathrm{Diag}(p) - p p^\top,
\]
the predictive Fisher information at $p$. Substituting into Equation~\eqref{eqn:restate_prop_ce} yields
\[
\mathbb{E}_\delta[\ell_{\mathrm{aug}}(z, y)]
\;\approx\;
\ell_{\mathrm{clean}}(z, y)
\;+\;
\frac{\alpha^2}{2K}\, \mathrm{Tr}\!\left( [\mathrm{Diag}(p) - p p^\top]\, J_z \Sigma_n J_z^\top \right),
\]
as claimed.
\end{proof}

\section{Additional Data Visualization}
\label{apx: data_vis}

\begin{figure}[H]
    \centering
    \includegraphics[width=\linewidth]{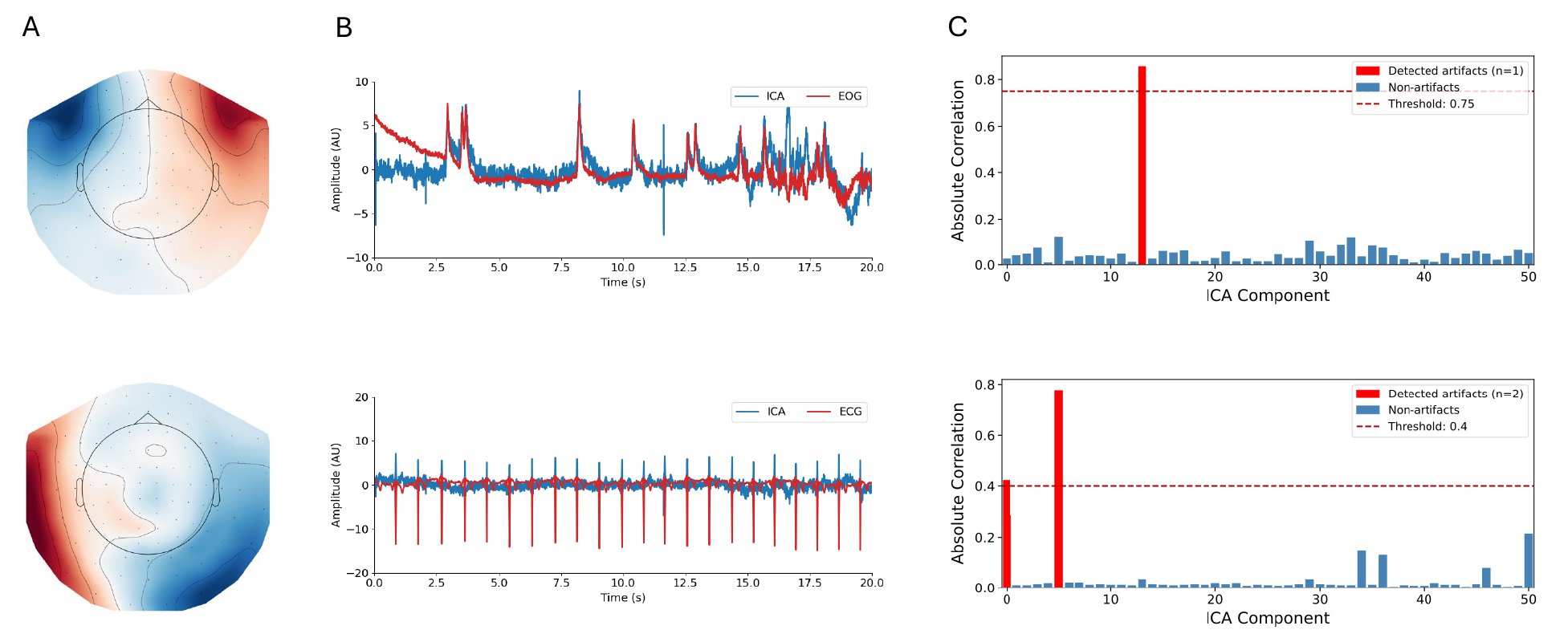}
    \caption{ICA artifact component selection for EOG and ECG signals. The topographies captured in column \textbf{A} show the relative strengths of the ICA mapping from source to sensor space for a sample EOG (top) and ECG (bottom) component. The corresponding ICA component waveforms are captured in column \textbf{B}, where they are compared to the reference EOG and ECG signals. Finally, column \textbf{C} show the absolute Pearson correlations between the first 50 ICA components and the artifact waveform for EOG (top) and ECG (bottom).}
    \label{fig:ICA-art}
    \end{figure}

\begin{figure}[H]
    \centering
    \includegraphics[width=0.8\linewidth]{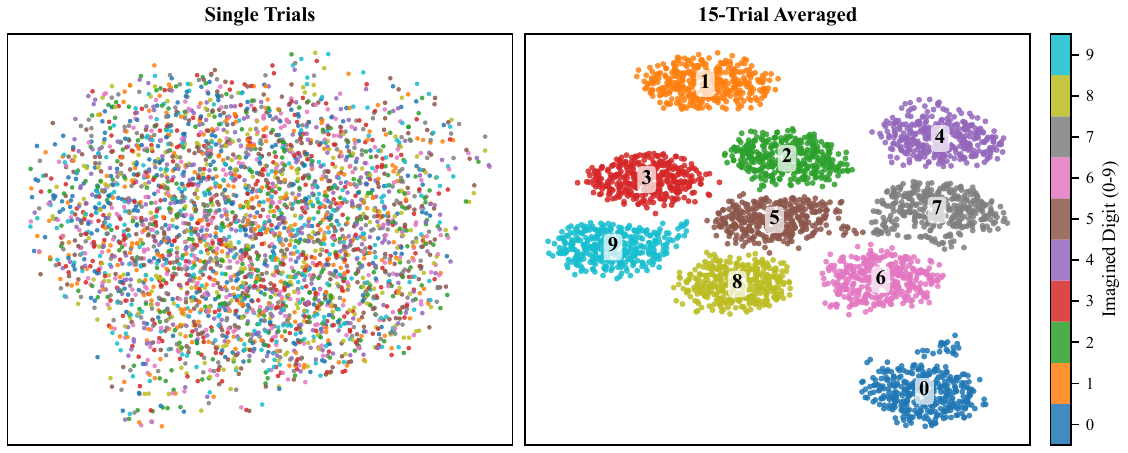}
    \caption{UMAP embeddings of intra-patient MEG recordings for imagined digits (0–9), shown for single trials (left) and 15-trial averages (right). Averaging improves cluster separability, and—unlike t-SNE, UMAP preserves aspects of global structure, revealing a separation between representations of digit 0 and digits 1–9}
    \label{fig:UMAP_classes}
\end{figure}

Unlike t-SNE, UMAP preserves some global structure, and the observed separation of digit 0 from digits 1–9 may reflect a systematic difference in neural representation rather than a purely local clustering artifact. This could arise from semantic differences between zero and non-zero numerals or from task-related strategies. While intgriguing, the effect is preliminary and must be validated across multiple subjects and modalities.

\begin{figure}[H]
    \centering
    \includegraphics[width=\linewidth]{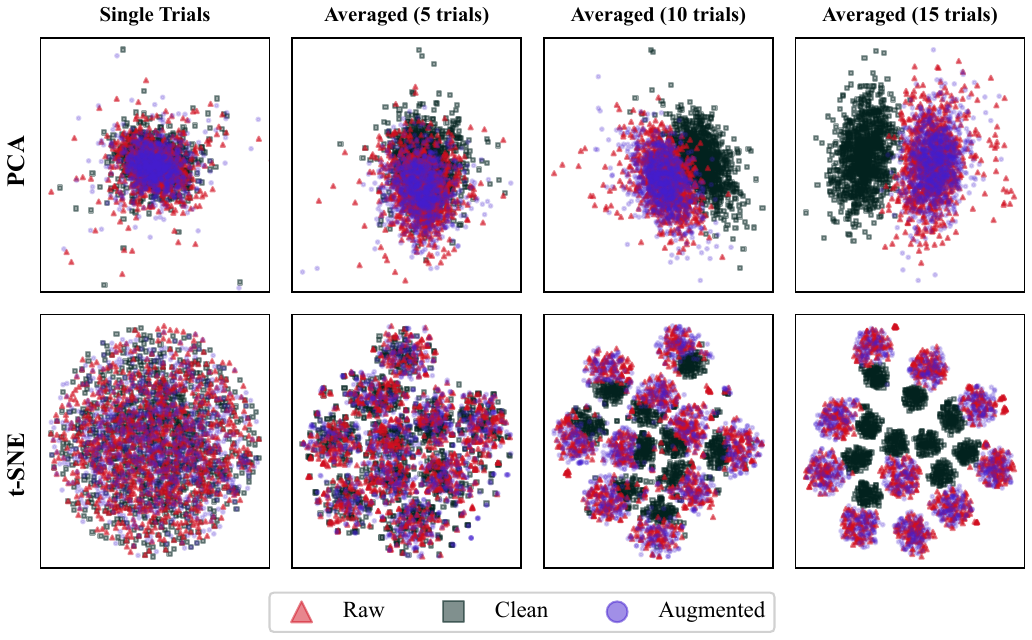}
    \caption{PCA and t-SNE embeddings of raw, augmented (ocular and cardiac), and clean data at various levels of averaging. While raw and augmented data remain closely aligned, the cluster of clean data separates with additional averaging, indicating that PNA correctly restores the corruption to artifact-cleaned data that would be seen at test time. Under t-SNE, label-wise clusters emerge for substantial averaging.}
    \label{fig:placeholder}
\end{figure}

\section{Model, Pre-processing and Baseline Augmentation Parameters}
\label{app: model_data_hyp}

\begin{table}[H]
    \centering
    \begin{tabular}{lc}
    \toprule
    \textbf{Parameter} & \textbf{Value} \\
    \midrule    
    \textbf{Averaging} & \\
    \# Samples Averaged ($K$) & 10 \\
    \# Training Samples (using resampling) & 10,000 \\
    \midrule    
    \textbf{Architectural Parameters} & \\
    MLP \# Hidden Layers & 1 \\
    MLP Hidden Layer Width & 128 \\
    EEGNet \# Feature Maps ($F_1$) & 16 \\
    EEGNet \# Spatial Filters per Feature Map ($D$) & 4 \\
    EEGNet \# Pointwise Filters ($F_2$) & 64 \\
    EEGNet Kernel Length & 25 \\
    \midrule
    \textbf{Training} & \\
    Batch Size & 64 \\
    Maximum \# Training Epochs & 200 \\
    Early Stopping & After 25 Epochs \\
    MLP Dropout Rate & 0.35 \\
    MLP Learning Rate & 0.0001 \\
    MLP Weight Decay & 0.0005 \\
    EEGNet Learning Rate & 0.0005 \\
    EEGNet Weight Decay & 0.001 \\
    Optimizer & AdamW~\cite{loshchilov2019adamw} \\
    \bottomrule
    \end{tabular}
    \caption{Model hyperparameters used in our experiments.}
    \label{tab:hyperparameters}
\end{table}

\begin{table}[H]
    \centering
    \begin{tabular}{lc}
    \toprule
    \textbf{Parameter} & \textbf{Value} \\
    \midrule
    \textbf{Input MEG Data} & \\
    \# Gradiometer Channels & 102 \\
    \# Magnetometer Channels & 204 \\
    Original Sampling Rate & 1000 Hz \\
    \midrule    
    \textbf{Preprocessing} & \\
Downsampled rate & 250 Hz \\
High-Pass Filter & 0.1 Hz \\
Low-Pass Filter & 120 Hz \\
Channel-wise Standardization & IQR = [-1, 1] \\
    \midrule
     \textbf{Data Splits} & \\
    Train : Val : Test & 10 : 1 : 1 \\
    \bottomrule
    \end{tabular}
    \caption{Preprocessing and data splits used on MegNIST data.  We apply channel-wise robust scaling to account for different magnitudes of gradiometer (T/m) and magnetometer (T) channels.}
    \label{tab:preprocessing}
\end{table}

\begin{table}[H]
    \centering
    \begin{tabular}{lc}
    \toprule
    \textbf{Parameter} & \textbf{Value} \\
    \midrule
    \textbf{Baseline Augmentation} & \\
    White Noise std & 0.1 \\
    Smooth Time Mask Length & 50 \\
    Frequency Shift & 0.5 Hz \\
    Temporal Shift & 1 step \\
    Amplitude Scaling & uniformly drawn from [0.9, 1.1] \\
    \bottomrule
    \end{tabular}
    \caption{Baseline augmentation hyperparameters, tuned to optimize single-trial validation set performance.}
    \label{tab:baseline-augmentation-hyperparameters}
\end{table}

\section{Additional Experimental Results}
\label{app: full_pna_results}

\begin{figure*}[h!t]
    \centering
    \includegraphics[width=\linewidth]{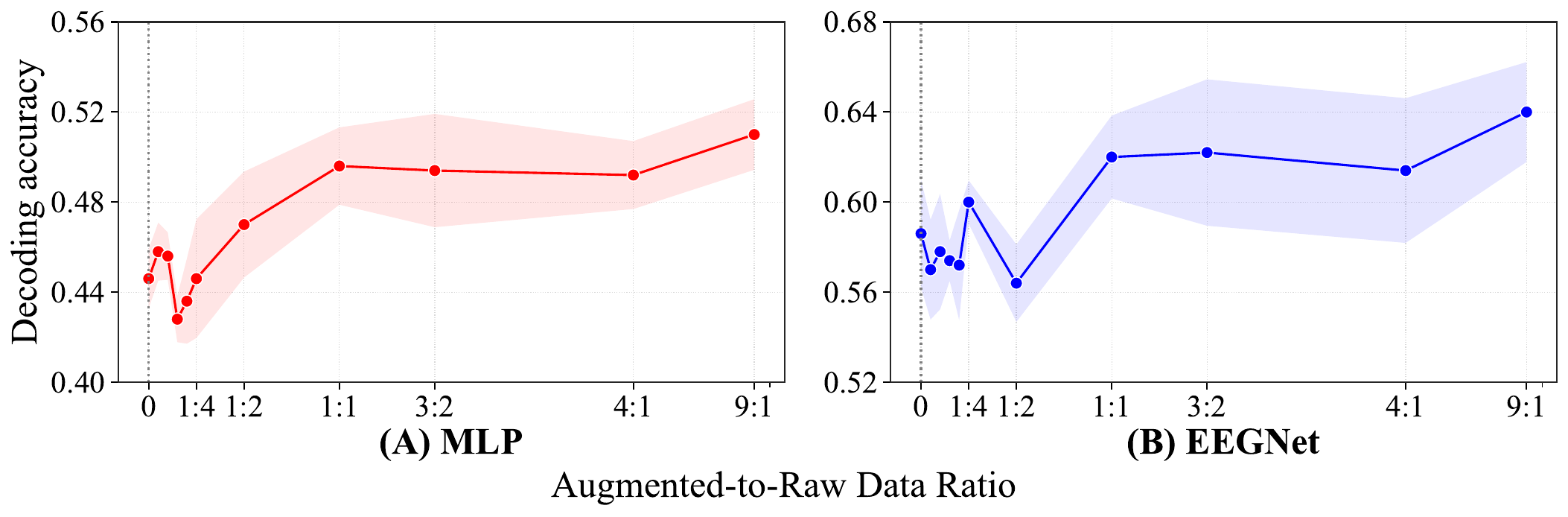}
    \caption{\textbf{Decoding accuracy scales with PNA samples (averaging without resampling).} Models are trained on a fixed budget of $10k$ raw trials ($1k$ post-averaging); the $x$-axis represents the ratio of augmented-to-raw data, ranging from 0 (baseline, vertical dotted line) to $90k$ ($9:1$ ratio; $9k$ post-averaging). Shaded regions denote $95\%$ confidence intervals ($n=5$ independent seeds); all runs 10-trial averaging.}
    \label{fig:averaging_without_resampling_results}
\end{figure*}

\begin{figure*}[h!t]
    \centering
    \includegraphics[width=\linewidth]{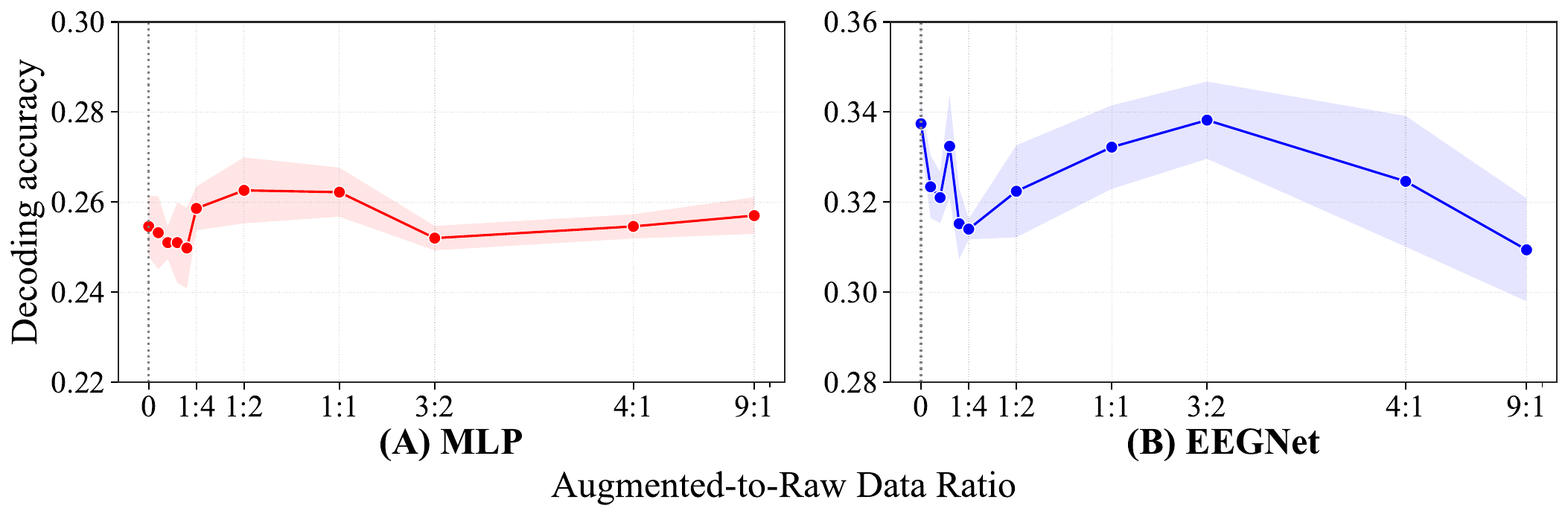}
    \caption{\textbf{Single-Trial Decoding accuracy under various PNA data infusions.} Models are trained on a fixed budget of $10k$ raw trials; the $x$-axis represents the ratio of augmented-to-raw data, ranging from 0 (baseline, vertical dotted line) to $90k$ ($9:1$ ratio). Shaded regions denote $95\%$ confidence intervals ($n=5$ independent seeds).}
    \label{fig:non_averaged_results}
\end{figure*}

For ``Raw + Aug.'', we concatenate varying amounts of augmented samples to 10,000 raw samples, to create final dataset sizes of 10,500, 11,000, 11,500, 12,000, 12,500, 15,000, 20,000, 25,000, 50,000 and 100,000.

\section{Compute Resources}
\label{app:compute-resources}

We used NVIDIA H100, L40S, RTX 8000, A100 and RTX A6000 GPUs with up to 96GiB of GPU memory and with up to 1TiB of RAM. We estimate that we used up to 5,000 hours of GPU compute over the course of this work, including final results (2,000 hours), hyperparameter searches (1,500 hours) and further experiments which were not included in the paper (500 hours).

\end{document}